%% file: neurips_2026.tex
\documentclass{article}

\usepackage[numbers]{natbib}
\usepackage[preprint]{neurips_2026}

\usepackage[utf8]{inputenc} 
\usepackage[T1]{fontenc}    
\usepackage{url}            
\usepackage{booktabs}       
\usepackage{amsfonts}       
\usepackage{nicefrac}       
\usepackage{microtype}      
\usepackage{xcolor}         
\newcommand{\ourMethod}{PRISM-LoRA} %

\definecolor{fhcolor}{rgb}{0.523, 0.235, 0.625}

\definecolor{citecolor}{HTML}{0071bc}
\usepackage[colorlinks=False, allcolors=citecolor]{hyperref}  

\usepackage{multirow}

\usepackage{amsthm,amsmath,amssymb}
\usepackage{bbm}
\usepackage{graphicx}

\usepackage{float}
\usepackage{subcaption}
\usepackage{enumitem}
\usepackage{etoc}

\usepackage{wrapfig}
\usepackage{multirow}
\usepackage[table]{xcolor}
\usepackage{fontawesome5}
\newtheorem{theorem}{Theorem}
\newtheorem{lemma}{Lemma}
\title{Principal-timestep Restricted Init via Sparse Matrix-decomposition in Flow-matching}
\author{%
    \textbf{Jiayang Gu\textsuperscript{1}\thanks{These authors contribute equally to this work.}},
    \textbf{Zheng Fang\textsuperscript{1}$^*$},
    \textbf{Lichuan Xiang\textsuperscript{1}},
    \textbf{Xu Cai\textsuperscript{2}},
    \textbf{Fanghui Liu\textsuperscript{1}},
    \textbf{Hongkai Wen\textsuperscript{1} \thanks{Correspondence.}}
\\
\\
    \textsuperscript{1}University of Warwick,
    \textsuperscript{2}Bytedance \\
}

\begin{document}

\maketitle

\begin{abstract}

Flow-matching diffusion models have recently emerged as a strong paradigm for high-fidelity visual generation. However, their prohibitively high fine-tuning cost limits scalability to downstream tasks. While Low-Rank Adaptation (LoRA) combined with spectral initialization has demonstrated accelerated convergence and improved performance in autoregressive language models by better aligning gradient directions, we find that it fails to deliver similar gains in diffusion fine-tuning, often yielding marginal or even negative improvements over vanilla LoRA.
We attribute this discrepancy to a fundamental mismatch between LoRA’s low-rank parameterization and the intrinsically high-rank gradients induced by the flow-matching objective. In particular, stochastic timestep sampling introduces directionally heterogeneous gradient signals across training steps, leading to misaligned updates under low-rank constraints.
To address this issue, we propose \ourMethod, a Principal-timestep Restricted Init via Sparse Matrix-decomposition framework that improves gradient alignment during fine-tuning. Our method consists of two key components: (i) principal timestep selection, which restricts initialization gradients to a subset of dominant timesteps to suppress effective gradient rank, and (ii) principal channel filtering, which removes task-irrelevant channels, enabling the one-step spectral initialization gradient to better align with the long-horizon optimization trajectory. 
Extensive experiments demonstrate that \ourMethod \ consistently improves both convergence speed and final performance across multiple diffusion fine-tuning benchmarks, including subject-driven generation, controllable generation, and deblurring, achieving not only performance improvement but also earlier stages of convergence over baseline LoRA and other spectral-init methods. Our code is available at \url{https://anonymous.4open.science/r/Prism-LoRA-28ED}.

\end{abstract}

\section{Introduction}
\label{sec:1-intro}
\input{chapters/1-intro}

\section{Related Work}

\input{chapters/2-relat}

\section{Method}
\input{chapters/4-metho}

\section{Experiment}
\input{chapters/5-expri}

\section{Conclusion}
\input{chapters/6-concl}

\begin{ack}
Use unnumbered first level headings for the acknowledgments. All acknowledgments
go at the end of the paper before the list of references. Moreover, you are required to declare
funding (financial activities supporting the submitted work) and competing interests (related financial activities outside the submitted work).
More information about this disclosure can be found at: \url{https://neurips.cc/Conferences/2026/PaperInformation/FundingDisclosure}.

Do {\bf not} include this section in the anonymized submission, only in the final paper. You can use the \texttt{ack} environment provided in the style file to automatically hide this section in the anonymized submission.
\end{ack}

{
\bibliographystyle{abbrvnat}
\bibliography{ref}
}

\appendix
\input{chapters/append}

\newpage
\input{checklist.tex}

\end{document}

%% file: chapters/1-intro.tex
\begin{figure}[ht]
    \centering
    \begin{subfigure}[t]{0.48\linewidth}
        \centering
        \includegraphics[width=\linewidth]{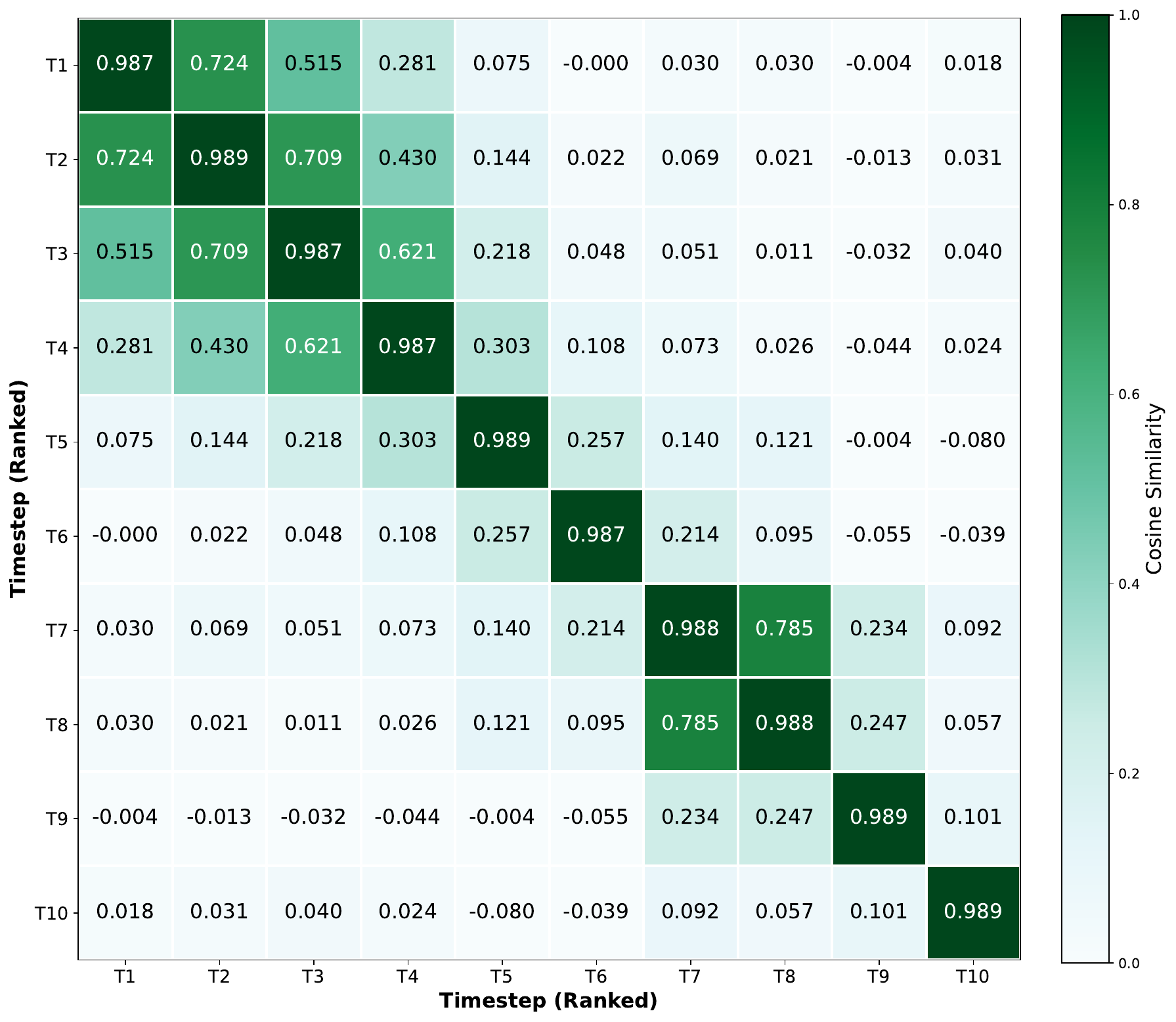}
        \caption{Gradient cosine similarity matrix.}
        \label{fig:heatmap}
    \end{subfigure}
    \hfill
    \begin{subfigure}[t]{0.48\linewidth}
        \centering
        \includegraphics[width=\linewidth]{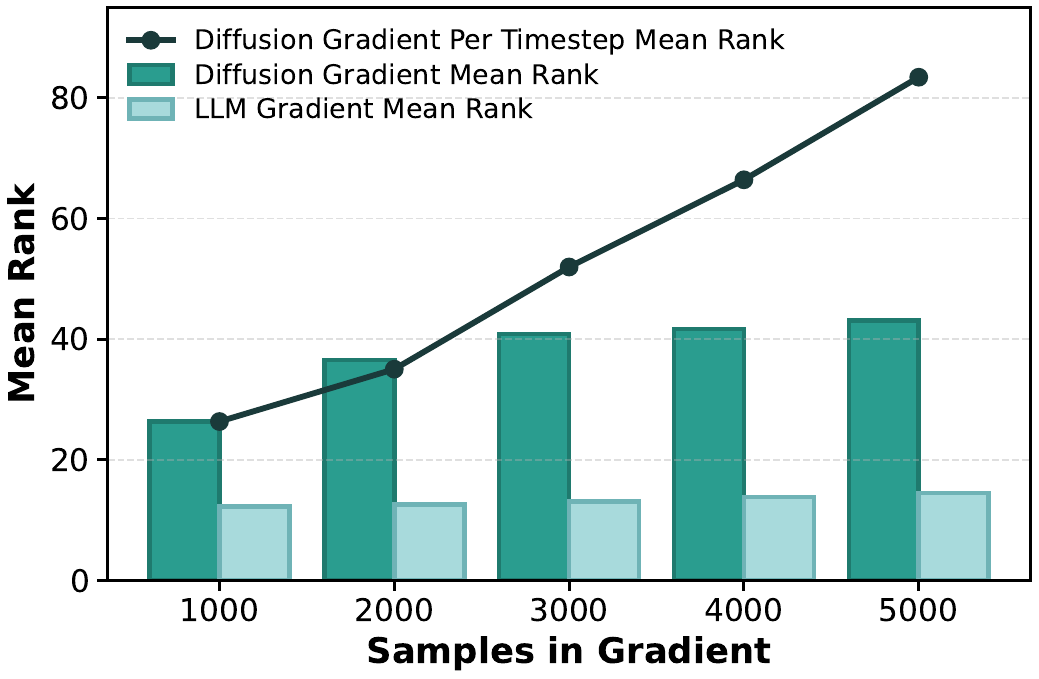}
        \caption{Mean rank comparison.}
        \label{fig:bars}
    \end{subfigure}
    \caption{Left figure shows the cosine similarity matrix on different diffusion timesteps, right figures show when sampled from a wider range of timesteps and sampling numbers, the difference in gradient rank and compared with LLM gradient.}
    \label{fig:first_figure}
    \vspace{-1.0em}
\end{figure}



Recently, the spectral initialization on LoRA has been receiving increasing research attention.
Other LoRA variants~\cite{lora-oft, luo2023lcm, soboleva2026t}, either require a dedicated design of hyperparameters or extra training parameters to improve the performance of fine-tuning.
In contrast, spectral initialization~\cite{lora-ga, lora-one, lora-sb} requires few hyperparameters and comes with tight theoretical guarantees on fine-tuning scaling. 
By initializing the LoRA adapters with the decomposition of either the pre-trained weights or a gradient obtained from a single-batch forward pass, it achieves both stronger performance and faster convergence.

The achievement of spectral initialization can be attributed to the successful transfer of the knowledge from the fine-tuning pattern of the full-parameter model to the low-rank parameters of LoRA.
Existing one-step gradient methods typically justify their effectiveness along two axes: (i) they show that the LoRA parameterization obtained by decomposing the first-step gradient attains the minimal approximation error relative to full fine-tuning at initialization; and (ii) they argue that this initialization remains well-aligned with the dominant gradient directions throughout subsequent training. Both arguments implicitly rely on a critical premise---that the \emph{effective rank} of the full-model gradient is comparable to the LoRA rank.
As effective rank is related to the energy distribution towards singular value, it indicates that using how many of the singular vectors on the front would have a reconstructive effect towards the original matrix.
If effective rank matches LoRA's rank, a good LoRA initialization weight derived from the decomposition result of the matrix can set an optimization start point similar as a full fine-tuning model does.

This premise, however, breaks down in the diffusion setting.
Figure~\ref{fig:first_figure} illustrates the issue from two angles.
Figure~\ref{fig:heatmap} shows that gradients from different timesteps share little similarity, indicating that they span largely distinct subspaces. Consequently, the aggregated gradient $G^\natural = \mathbb{E}_t[G_t]$ exhibits a substantially higher effective rank than any single $G_t$.
Figure~\ref{fig:bars} quantifies this effect. We define the empirical effective rank as the number of top singular vectors needed to reconstruct 99\% of the gradient's Frobenius norm, and progressively enlarge the aggregated timestep range (``1000'' covers $T_1$--$T_2$, ``2000'' covers $T_1$--$T_4$, and so on). Compared to an LLM baseline, the diffusion gradient's effective rank grows markedly faster as more timesteps are aggregated; the overlaid line further shows that this rank itself rises as timesteps move from noise to data.
As a result, $G^\natural$ inherits an inflated rank structure, and the low-rank SVD on which one-step methods like LoRA-One rely becomes a lossy operation in the diffusion setting.

We attribute this failure to a fundamental \textbf{mismatch between spectral-init's low-rank assumption and the gradient geometry of diffusion training}. In LLMs, the one-step gradient is well-approximated as low-rank, justifying spectral initialization. In diffusion, the one-step gradient aggregates over a sequence of timesteps $\mathbb{E}_t[G_t]$; since different t correspond to denoising tasks at different signal-to-noise ratios, the $G_t$ point in widely divergent directions. The aggregated gradient, therefore, has substantially higher effective rank, violating the core premise of spectral initialization.

On this basis, we raise 2 hypotheses corresponding to 2 aspects, shown as follows:
\paragraph{Hypothesis 1}(Principal Timesteps reduce effective rank) Restricting gradient estimation to small timestep t(closer to noise) $t\rightarrow0$ yields a gradient matrix $G^\natural$ with substantially lower effective rank than uniform timestep sampling, improving the fidelity of ranking-$r$ spectral initialization.
\paragraph{Hypothesis 2}(Sparse decomposition improves finite-sample subspace recovery) The irrelevant channel signal would hinder the reconstructed weight from being aligned with the fine-tuning orientation.

Corresponding to Hypothesis 1, we introduce a principal timestep selection strategy that confines gradient estimation to a small set of timesteps, suppressing the effective rank of the guidance gradient before decomposition. Corresponding to Hypothesis 2, we apply a sparse channel decomposition that isolates task-relevant channels from the gradient signal, so that the reconstructed low-rank update aligns with the long-horizon fine-tuning trajectory rather than being polluted by irrelevant directions.
Our contributions are summarized as follows:

1. We identify and analyze the failure mode of spectral initialization in diffusion fine-tuning, attributing it to a mismatch between the low-rank assumption of spectral initialization and the high-rank gradient geometry induced by stochastic timestep sampling.

2. Motivated by two hypotheses on gradient rank and channel relevance, we propose \ourMethod, which combines principal timestep selection and sparse channel decomposition to produce a low-rank, well-aligned initialization for diffusion LoRA.

3. We validate our method on four downstream tasks (Dreambooth, Canny, Depth, Deblur) across Stable Diffusion-3 and Flux-1, achieving not only performance improvement but also training acceleration at the initial training stage over strong spectral-init baselines.



%% file: chapters/2-relat.tex
\paragraph{Diffusion-based Generative Models.}
The diffusion model was originally introduced by \cite{sohl2015deep}, which has replaced Generative Adversarial Network (GAN)-based methods \cite{wang2024spatial,yan2024orthogonal,mei2024holo}, and has been widely applied in the field of image synthesis \cite{hua2023dreamtuner,fangflexcontrol,zhang2025scaling}. Compared to adversarial optimization approaches, which easily lead to training instability and even collapse, diffusion models utilize the probabilistic optimization scheme to ensure stable convergence and generation diversity. The Latent Diffusion Model (LDM) \cite{rombach2022high} reduces computational demands by transferring the diffusion process from pixel space to latent space. Based on it, Transformer-based architectures such as Diffusion Transformer (DiT) \cite{pu2024efficient,zhao2025dydit++,chen2025unireal,jia2025d,cai2025shortcutting}, Stable Diffusion 3.0/3.5 \cite{SD3}, FLUX \cite{flux} further extend diffusion's capacity to model long-range dependencies and scale to massive multi-modal datasets. Meanwhile, Rectified Flow optimization \cite{lipman2022flow,liu2022flow,liu2023instaflow} enables faster sampling with fewer denoising steps without degrading quality. Although the merits are obvious, pre-training a diffusion model is undoubtedly expensive.
\vspace{-0.5em}

\paragraph{Post-training of Diffusion Models.} 
Compared to pre-training, fine-tuning the fully-trained diffusion model on downstream task with a small-scale dataset \cite{wu2025difix3d+,ji2026lidarpainter,chen2025ultrafusion} can be more feasible and cost-friendly. DreamBooth \cite{ruiz2023dreambooth} finetunes text-to-image diffusion models using a few subject-specific images provided by practitioners, binding the subject to a unique identifier token to enable high-fidelity and controllable personalized generation. \cite{zhang2023adding,mou2024t2i,peng2024controlnext} freeze the denoising backbone and finetune an additional module to transfer conditional controls into the noise latent space, thereby achieving spatially-aligned generation. Meanwhile, due to the decoupled design with the denoising network—parameter updation occurring only on the trainable adapter, the training stability and convergence speed are both superior to the full-parameter fine-tuning. To further compress the trainable parameters, \cite{tan2024ominicontrol,tan2025ominicontrol2,zhang2025easycontrol,wang2025unicombine,fangDynFusion} introduce efficient LoRA module \cite{hu2022lora} for diffusion fine-tuning, which not only reduces the demand for gpu memory, but also makes it convenient for the diffusion model to switch freely between different design tasks. However, the number of update steps required for fine-tuning is still quite large, and the performance is difficult to match that of full-parameter fine-tuning.
\vspace{-0.5em}

\paragraph{Improving Fine-tuning Efficiency.}
Low-Rank Adaptation (LoRA) \cite{hu2022lora,dettmers2023qlora,luo2023lcm,soboleva2026t} has dominated efficient large vision/language models' fine-tuning. Apart from the parameter efficiency, researchers have also extended to explore time efficiency. On the one hand, some research works focus on improving initialization strategies to enhance the effectiveness of low-rank updates. \cite{lora-sb,lora-ga,lora-one} propose an initialization method based on gradient update approximation, which enables the low-rank parameters to approach the optimal update direction at the early stage of training. \cite{zi2023delta,pissa,liu2024dora} leverage singular value decomposition (SVD) \cite{weiland2009singular} to confine the adaptation within the principal singular subspace, thereby enabling more stable and robust initialization and updates of model weights by better aligning with the intrinsic structure of the pretrained model. On the other hand, there are also efforts to improve training efficiency by focusing on the optimization process. \cite{hayou2024lora+,huang2024allora,zhang2023lora,zhao2024galore} allocate different learning rates to different matrices in low-rank decomposition, significantly accelerating convergence and improving performance without changing the initialization method. However, due to the multi-optimization training objectives in diffusion fine-tuning, these methods fail to capture a robust enough gradient by only relying on one-step gradient to guide the overall training process.

%% file: chapters/4-metho.tex
In this section, we will introduce how we solve the new problem of the gradient-based method applied in the diffusion domain.
We understand that fluctuating training loss is the inevitable nature of fine-tuning the diffusion model, and it can hardly ensure that the 1-step optimization direction points at the average optimization direction of the overall fine-tuning process.
We notice this problem and simplify it in Sec. \ref{sec:4-1}, defined as \textit{High-ranking $G^\natural$ problem}. It can be solved in 2 aspects, the timestep and channel, shown in Sec. \ref{sec:principal_timestep} and Sec. \ref{sec:golden_channel}, respectively.
Finally, we formulate the overall training process in Sec.~\ref{sec:spectral_init}, to have an overall look at our method.
For preliminaries, see Appendix.~\ref{sec:supp_prelim}.

\begin{figure}[htbp]
    \centering
    \begin{subfigure}[b]{0.24\textwidth}
        \centering
        \includegraphics[width=\textwidth]{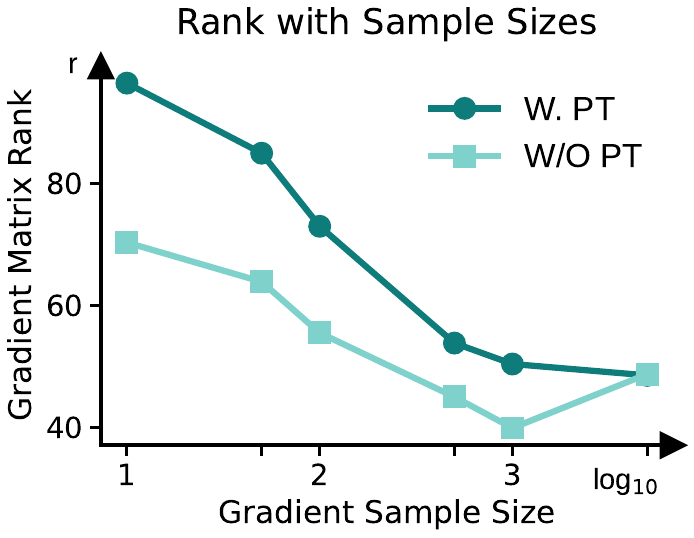}
        \caption{}
        \label{fig:sub_a}
    \end{subfigure}
    \hfill
    \begin{subfigure}[b]{0.24\textwidth}
        \centering
        \includegraphics[width=\textwidth]{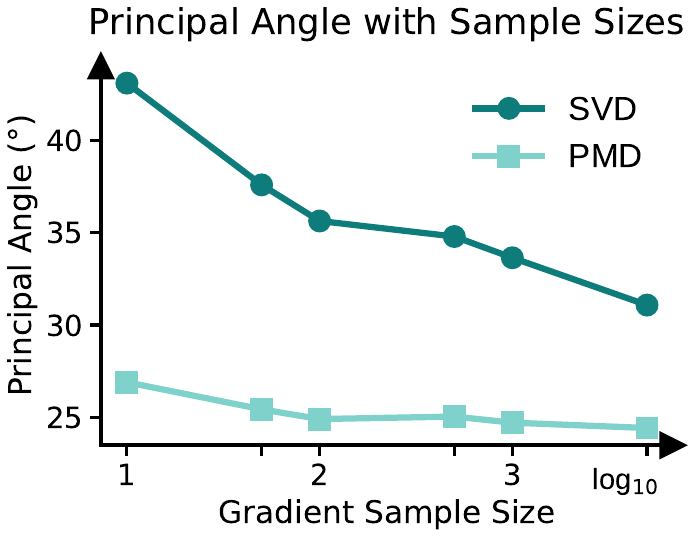}
        \caption{}
        \label{fig:sub_b}
    \end{subfigure}
    \hfill
    \begin{subfigure}[b]{0.24\textwidth}
        \centering
        \includegraphics[width=\textwidth]{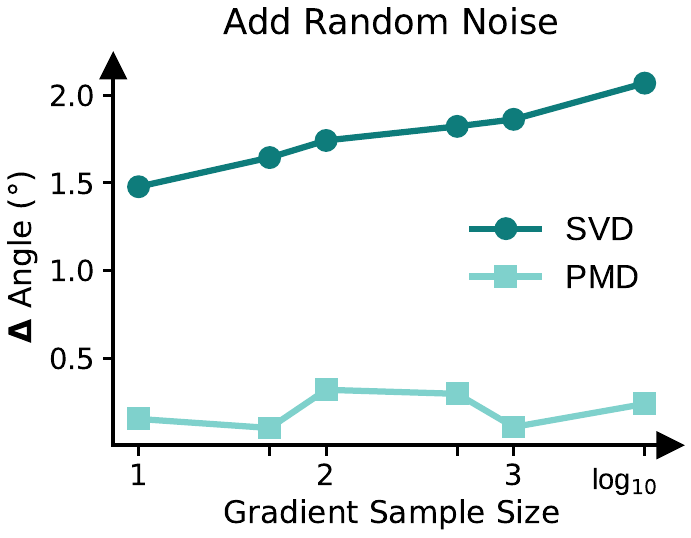}
        \caption{}
        \label{fig:sub_c}
    \end{subfigure}
    \hfill
    \begin{subfigure}[b]{0.24\textwidth}
        \centering
        \includegraphics[width=\textwidth]{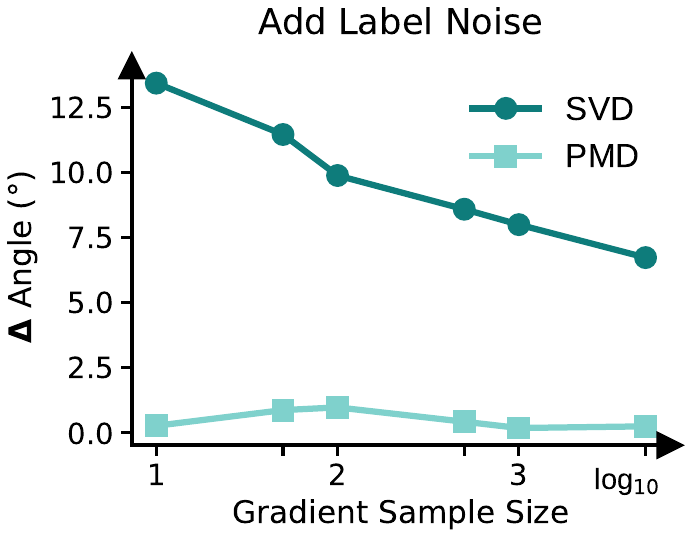}
        \caption{}
        \label{fig:sub_d}
    \end{subfigure}
    \caption{\textbf{Empirical validation of our proposed hypotheses.}
    Fig.~\ref{fig:sub_a} indicates that the principal timestep(PT) achieves a lower rank than stochastic timesteps sampling.
    Fig.~\ref{fig:sub_b} under principal timestep sampling, calculates the principal angle with $G$ with $50K$ sample size, achieving a closer direction.
    Fig.~\ref{fig:sub_c} and ~\ref{fig:sub_d} shows golden channel module has a better ability in filtering out irrelevant channel information compared with SVD. For further analysis, see Sec.~\ref {sec:golden_channel}.
    }
    \label{fig:method_intuitive}
    \vspace{-1.0em}
\end{figure}

\subsection{Barriers in Flow-matching Diffusion Model}
\label{sec:4-1}

\paragraph{High-ranking $G^{\natural}$.}
Near-orthogonal per-timestep gradients do not lead to a low-rank average, as each $G_t$ contributes an independent direction to the column or row space of $G^\natural$.
Formally, if $G_{t_1}$ and $G_{t_2}$ have near-orthogonal top-$r$ subspaces, the rank of their sum approaches $rank(G_{t_1})+rank(G_{t_2})$. 
Averaging over a continuous number of timesteps therefore inflates the effective rank of $G^\natural$ relative to any single $G_t$.
From a mathematical perspective, average gradient from different directions would lead to a high-ranking $G^{\natural}$.
It has 2 drawbacks: The decomposed LoRA parameter from high-ranking $G^{\natural}$ has a conflict with the low-ranking training objective with LoRA; High-ranking $G^{\natural}$ may have a deviated guidance on multiple training objects.
However, directly measuring the true rank of the gradient matrix is challenging. To address this, we convert this problem into measuring the spectral of the gradient matrix, which would have a positive correlation with the actual rank $r^*$ of $G^\natural$.
\vspace{-0.8em}

\paragraph{From gradient rank to posterior covariance spectrum.}
Directly analyzing the spectrum of $G^\natural$ is intractable, as it entangles the data-driven residual with the architecture-dependent Jacobian $J_\theta$. Under a locally linearized view, the per-timestep gradient covariance decomposes as
\begin{equation}
G_t \;\approx\; \mathbb{E}_{x_t}\!\left[\, J_\theta^\top \, \Sigma_v(t) \, J_\theta \,\right],
\end{equation}
where $\Sigma_v(t) = \tfrac{1}{t^2}\mathrm{Var}[x_0 \mid x_t]$ is the velocity-regression residual covariance. 
Since the architecture-dependent factor $J_\theta$ admits no general treatment, and since an ill-conditioned output target necessarily induces an ill-conditioned $G_t$, we adopt $\Sigma_v(t)$ as a principled proxy for the spectral behavior of $G^\natural$. Furthermore, as absolute rank is numerically unstable, we replace it with a scale-invariant surrogate---the concentration of the spectrum---which is positively correlated with the effective rank and amenable to rigorous comparison via majorization.

We measure spectral concentration via the $r_{N}$ metric: the minimum number of top eigenvalues needed to capture $N$\% of the total variance, $r_{N}(\Sigma) = \min\{r: \sum_{i=1}^{r} {\sigma}_i \ge N\%\}$ where ${\sigma}_i$ are singular values. Small $r_{N}$ indicates a sharp, concentrated spectrum (effectively low-rank); $r_{N} \to d$ indicates a flat spectrum (high-rank).



\subsection{Principal Timestep}
\label{sec:principal_timestep}
The effectiveness of the principal timestep is based on the following justification:
the spectrum of the one-step gradient covariance $\Sigma_v(t)$ becomes flatter when mixing with more close-to-target timestep.
As a result, restricting spectral initialization to smaller $t$ leads to top-$r$ subspace capturing a substantially larger fraction of the total signal.
We formulate this theorem as follows:
\begin{theorem}[Monotone spectral flattening]
\label{thm:monotone-flattening}
Under the assumptions of Sec.~\ref{sec:4-1}, the rank-N index $r_{N}$ 
of the one-step gradient covariance $G^{\natural}$ is monotonically 
non-decreasing in $t$ on $[0,1]$:
\begin{equation}
    r_{N}(\Sigma_v(0)) \;\leq\; r_{N}(\Sigma_v(0.5)) 
    \;\leq\; r_{N}(\Sigma_v(1)) \;=\; \lceil N / 100, d \rceil.
\end{equation}
\end{theorem}

By the flow-matching identity $x_t = (1-t)x_0 + t x_1$, we eliminate $x_1$ via $x_1 = \bigl(x_t - (1-t)x_0\bigr)/t$, so that the posterior covariance $\mathrm{Var}[x_0 \mid x_t]$ depends on $x_0$ alone. Diagonalizing in the eigenbasis of $\Sigma_1$, the $i$-th singular value of $\Sigma_v(t)$ admits the closed form
\begin{equation}
\label{eq:definition}
    \sigma_i(t) \;=\; \frac{\lambda_i}{(1-t)^2 + t^2 \lambda_i},
\end{equation}
where $\{\lambda_i\}_{i=1}^{d}$ are the singular values of the data covariance $\Sigma_1 = \mathrm{Cov}(x_1)$, and by symmetric positive semi-definiteness coincide with its eigenvalues. A full derivation is deferred to Appendix~\ref{supp:sec:problem_formulation}.

\paragraph{Two-Ends Comparison.} We first examine the two endpoints $t \to 0$ and $t \to 1$. At $t \to 0$, $\sigma_i(0) = \lambda_i$, recovering the spectrum of $\Sigma_1$; at $t = 1$, $\sigma_i(1) = 1$ for all $i$, yielding an isotropic covariance $\Sigma_v(1) = I_d$. 
To compare these two spectra, we treat $\sigma_i(t) = f_t(\lambda_i)$ as a univariate map with $t$ as a fixed parameter and $\lambda_i$ as the input variable. A direct computation shows that $f_t$ is strictly increasing, strictly concave, and satisfies $f_t(0^+) = 0$. By Lemma~\ref{lem:majorization} (normalized-spectrum majorization), these three properties imply that the normalized spectrum of $\Sigma_v(t)$ is \emph{majorized} by that of $\Sigma_1$---that is, the concave map $f_t$ compresses large singular values and lifts small ones, producing a strictly flatter output sequence. Consequently, every Schur-concave flatness functional (including $r_{N}$) satisfies $r_{N}(\Sigma_v(t)) \geq r_{N}(\Sigma_1)$ for any $t \in (0,1)$, with strict inequality whenever $\Sigma_1$ has a non-degenerate spectrum. Detailed arguments for these two endpoints are given in Appendices~\ref{supp:sec:spectial_t_0} and~\ref{supp:sec:spectial_t_1}.

Having established flattening relative to the endpoint $t = 0$, we next strengthen the result to \emph{full-range monotonicity}: the normalized spectrum $\{\sigma_i(t)\}_{i=1}^{d}$ becomes strictly flatter as $t$ increases on the entire interval $[0,1]$, not merely at the endpoints. This is established in the following subsection by a second application of Lemma~\ref{lem:majorization} to the transition map between arbitrary time pairs $t_1 < t_2$.

\paragraph{Full-Range Monotonicity.}
\label{sec:full-range-monotonicity}

We now strengthen Theorem~\ref{thm:monotone-flattening} from endpoint comparison to full-range monotonicity. Fix $0 \leq t_1 < t_2 \leq 1$ and define the transition map $h_{t_1 \to t_2}: (0, \infty) \to (0, \infty)$ that sends each singular value at time $t_1$ to its counterpart at time $t_2$:
\begin{equation}
\label{eq:transition-map}
    h_{t_1 \to t_2}(s) \;\triangleq\; 
    f_{t_2}\!\bigl(f_{t_1}^{-1}(s)\bigr), 
    \qquad s \in \bigl\{\sigma_i(t_1)\bigr\}_{i=1}^{d},
\end{equation}
so that $\sigma_i(t_2) = h_{t_1 \to t_2}(\sigma_i(t_1))$ for every $i$.

A direct calculation (Appendix~\ref{supp:sec:transition-map}) shows that $h_{t_1 \to t_2}$ is strictly increasing, strictly concave, and satisfies $h_{t_1 \to t_2}(0^{+}) = 0$. Applying Lemma~\ref{lem:majorization} to $h_{t_1 \to t_2}$ with input sequence $\{\sigma_i(t_1)\}_{i=1}^{d}$ yields
\begin{equation}
    \widehat{\sigma}(t_1) \;\succ\; \widehat{\sigma}(t_2),
\end{equation}
where $\widehat{\sigma}(t) = \sigma(t) / \sum_{j} \sigma_j(t)$ denotes the normalized spectrum. Since every Schur-concave flatness functional respects the majorization order, and $r_{N}$ is Schur-concave on the simplex of normalized spectra, we conclude
\begin{equation}
    r_{N}\bigl(\Sigma_v(t_1)\bigr) \;\leq\; r_{N}\bigl(\Sigma_v(t_2)\bigr), 
    \qquad \forall\, 0 \leq t_1 < t_2 \leq 1,
\end{equation}
with strict inequality whenever $\Sigma_1$ has a non-degenerate spectrum. 
This completes the proof of Theorem~\ref{thm:monotone-flattening}.

In the Fig.~\ref{fig:method_intuitive}, the $r_{N}$ line using SVD method to decompose $G^\natural$, and the sim score is calculated by similarity with the actual fine-tuning delta weights subtracted by the final fine-tuning weights and pre-training weights.

Sec.~\ref{sec:full-range-monotonicity} characterizes the spectrum of the input-space conditional covariance $\Sigma_v(t)$, which encodes the posterior uncertainty over $x_0$ given $x_t$. The empirical gradient that drives spectral initialization, however, is the weight gradient $\hat{G}_t = \frac{1}{N}\sum_i e^{(i)}_t (h^{(i)}_t)^\top$, where $e^{(i)}_t$ is the per-sample prediction error and $h^{(i)}_t$ is the backbone feature. The effective rank of $\hat{G}_t$ is therefore governed by the diversity of $\{h^{(i)}_t\}_{i=1}^{N}$ across samples, rather than by $\Sigma_v(t)$ alone. The two quantities nonetheless align at the data endpoint: as $t \to 0$, inputs $x_t \to x_0$ concentrate near the data manifold, so backbone features collapse to a low-dimensional response region and $\mathrm{rank}_{\mathrm{eff}}(\hat{G}_t)$ drops in tandem with the sharpening of $\Sigma_v(t)$. Sec.~\ref{sec:full-range-monotonicity} thus provides the input-level justification for the data-side endpoint of our Principal Timesteps; the symmetric low-rank behavior at the noise-side endpoint $t \to 1$ arises from a complementary mechanism (feature collapse under near-isotropic inputs) that we examine empirically in Section~\ref{sec:1-intro}.



\subsection{Golden Channel}
\label{sec:golden_channel}

Though $r_{N}^{G^\natural}$ has decreased to a level that can be reconstructed well by SVD under the LoRA rank budget, two concerns remain. First, the irrelevant information carried by $x_t$ may still contaminate $G^\natural$, so the leading singular directions of $G^\natural$ are not guaranteed to align with the true optimization direction of the fine-tuning trajectory. Second, since the gradients we use for reinitialization are estimated from a finite sample of $x_t$, SVD---which fits the dominant energy of \emph{this particular sample}---is sensitive to sample-specific noise rather than to the underlying shared signal across samples.

\paragraph{From SVD to PMD.}
To address both issues, we replace SVD with Penalized Matrix Decomposition (PMD)~\cite{pmd} when extracting the low-rank subspace of $G^\natural$. PMD reformulates the leading singular component as a constrained bilinear maximization with additional sparsity penalties on the factors:
\begin{equation}
\label{eq:pmd}
    \max_{u, v} \; u^\top G\, v \quad \text{s.t.} \quad \|u\|_2 \le 1,\; \|v\|_2 \le 1,\; \|u\|_1 \le c_1,\; \|v\|_1 \le c_2,
\end{equation}
and subsequent components are obtained by deflation. When $c_1, c_2$ are large enough, Eq.~\ref{eq:pmd} reduces to the variational form of SVD; when they are tightened, the recovered factors are encouraged to concentrate on a small set of coordinates rather than spreading energy uniformly across all dimensions. This sparsity prior is well aligned with our setting: the true fine-tuning direction is expected to live on a low-rank \emph{and} structurally compact subspace of the parameter space, while the noise introduced by irrelevant $x_t$ tends to be diffuse. By penalizing diffuse factors, PMD is biased away from sample-specific fluctuations and toward directions that are consistently expressed across gradient samples. We therefore expect PMD to yield a $G^\natural$ that is (i) better aligned with the large-sample limit and (ii) more robust to perturbations on individual samples.

\paragraph{Empirical verification.}
We validate the golden channel with the experiment summarized in Fig.~\ref{fig:method_intuitive}. We treat the gradient estimated from a large number of $x_t$ samples ($50000$ iters) as a proxy for the true optimization direction, and measure how well the rank-$r$ subspace recovered from a small-sample gradient aligns with it, using the arc sine of the principal angle (smaller is better). We compare two decompositions, SVD and PMD, in two regimes: clean small-sample gradients, and under two kinds of noise (Gaussian noise sampled from a Gaussian distribution and random label noise sampled from another dataset).

Two observations support the use of PMD. \emph{(i) Better alignment.} Across all sample budgets, the PMD subspace is markedly closer to the large-sample direction than the SVD subspace; for instance, with only $10$ iters PMD already attains a closer angle compared with the $10^4$ sampled SVD result. \emph{(ii) Better noise robustness.} Under additive Gaussian and sample from another dataset perturbation, SVD degrades substantially, whereas PMD remains nearly unchanged, indicating that the PMD subspace is governed by the cross-sample shared signal rather than by sample-specific fluctuations. 

Together, these results confirm that PMD recovers a $G^\natural$ that is both more faithful to the true fine-tuning direction and more stable under noise---exactly the properties we need for a reliable reinitialization of LoRA.



\subsection{LoRA Initialization via Spectral Decomposition}
\label{sec:spectral_init}

The training pipeline of \ourMethod \ consists of three stages: \emph{gradient estimation}, \emph{parameter initialization}, and \emph{fine-tuning}.
We will introduce the process of each stage in the following.

\paragraph{Stage 1: Gradient Estimation.}
In the first stage, we collect $M$ gradient samples by performing forward-backward passes under the principal timestep distribution:
\begin{equation}
\label{eq:estimate_gradient}
G^\natural = \frac{1}{M} \sum_{i=1}^{M} \nabla_\theta \mathcal{L}_{\text{FM}}\bigl(\theta; x_0^{(i)}, x_1^{(i)}, t_i\bigr), \quad t_i \sim p^\natural(t),\; x_0^{(i)} \sim \mathcal{N}(0, I),\; x_1^{(i)} \sim p_{\text{data}},
\end{equation}
where $\theta$ denotes the (frozen) base-model parameters of the target weight matrix, $M$ is the number of gradient samples, $\mathcal{L}_{\text{FM}}$ is the flow-matching loss defined in Eq.~\eqref{eq:training_loss}, $p^\natural(t)$ is the principal timestep distribution supported on a low-$t$ subinterval of $[0,1]$, $x_0^{(i)}$ and $x_1^{(i)}$ are the noise and data endpoints of the $i$-th sample, and $G^\natural$ is the resulting low-rank gradient estimate on which the subsequent decomposition operates.

\paragraph{Stage 2: Parameter Initialization.}
In the second stage, we apply the golden-channel sparse SVD to $G^\natural$ and initialize the LoRA factors from its leading spectral components:
\begin{equation}
\label{eq:spectral_init}
A_0 = \sqrt{\gamma}\, \bigl[U_{G^\natural}\bigr]_{[:,\,1:r]} \bigl[S_{G^\natural}^{1/2}\bigr]_{[1:r]}, \quad B_0 = \sqrt{\gamma}\, \bigl[S_{G^\natural}^{1/2}\bigr]_{[1:r]} \bigl[V_{G^\natural}\bigr]_{[:,\,1:r]}^\top,
\end{equation}
where $U_{G^\natural}$, $S_{G^\natural}$, and $V_{G^\natural}$ are the left singular vectors, singular values, and right singular vectors produced by the golden-channel sparse decomposition of $G^\natural$; $r$ is the target LoRA rank; $[\cdot]_{[:,\,1:r]}$ and $[\cdot]_{[1:r]}$ denote truncation to the top-$r$ components; $\gamma$ is the LoRA scaling factor; and $A_0 \in \mathbb{R}^{d_{\text{out}} \times r}$ and $B_0 \in \mathbb{R}^{r \times d_{\text{in}}}$ are the initialized LoRA down- and up-projection matrices, respectively.

\paragraph{Stage 3: Fine-tuning.}
In the final stage, we fine-tune the LoRA factors $(A, B)$ initialized from Eq.~\eqref{eq:spectral_init} using the standard flow-matching objective, without any further modification to the timestep distribution:
\begin{equation}
\label{eq:training_loss}
\mathcal{L}_{\text{FM}}(\theta) = \mathbb{E}_{t \sim \mathcal{U}(0,1),\, x_0 \sim \mathcal{N}(0, I),\, x_1 \sim p_{\text{data}}} \Bigl[\bigl\| v_\theta(x_t, t) - (x_1 - x_0) \bigr\|_2^2\Bigr], \quad x_t = (1-t)x_0 + t x_1,
\end{equation}
where $v_\theta(\cdot, \cdot)$ is the flow-matching velocity predictor parameterized by the base weights and the trainable LoRA factors, $t \sim \mathcal{U}(0,1)$ is the training timestep sampled uniformly over the full interval, $x_t$ is the linear interpolant between noise $x_0$ and data $x_1$, and $x_1 - x_0$ is the target velocity field. During this stage, the base parameters are kept frozen, and only $(A, B)$ are updated.

\begin{table}[h]
\centering
\caption{\textbf{Quantitative comparison with existing methods} on different text-to-image generation tasks. The \textbf{bold} and \underline{underlined} figures represent the optimal and sub-optimal results, respectively.}
\label{tab:main_result}
\resizebox{\textwidth}{!}{
\begin{tabular}{l|l|cc|ccc|ccc}
\toprule
\multirow{2}{*}{Task} & \multirow{2}{*}{Methods / Setting} & \multicolumn{2}{c|}{Controllability} & \multicolumn{3}{c|}{Alignment} & \multicolumn{3}{c}{Image Quality} \\
 & & F1$\uparrow$ & MSE$\downarrow$ & CLIP Text$\uparrow$ & CLIP Image$\uparrow$ & DINO$\uparrow$ & FID$\downarrow$ & SSIM $\uparrow$ & PSNR $\uparrow$ \\
\midrule
\multirow{5}{*}{Dreambooth} 
    & SD3-medium                & - & - & \underline{0.224} & 0.793 & \underline{0.667} & 170.0 & - & - \\
    & \quad + Pissa             & - & - & 0.104 & 0.475 & 0.024 & 418.0 & - & - \\
    & \quad + LoRA-GA           & - & - & 0.168 & 0.699 & 0.533 & 260.0 & - & - \\
    & \quad + LoRA-One          & - & - & \underline{0.224} & \underline{0.795} & \underline{0.667} & \underline{169.1} & - & - \\
    & \quad + \ourMethod        & - & - & \textbf{0.230} & \textbf{0.805} & \textbf{0.682} & \textbf{160.3} & - & - \\
\midrule
\multirow{5}{*}{Canny} 
    & OminiControl              & 0.4763 & - & 0.2203 & 0.7085 & 0.5011 & 97.56 & \textbf{0.3702} & 8.57 \\
    & \quad + Pissa             & 0.0588 & - & 0.1466 & 0.4964 & -0.0166 & 456.62 & 0.2297 & 5.67 \\
    & \quad + LoRA-GA           & 0.4305 & - & 0.2128 & 0.6802 & 0.4552 & 108.25 & 0.2397 & 7.34 \\
    & \quad + LoRA-One          & \underline{0.4893} & - & \textbf{0.2232} & \underline{0.7214} & \underline{0.5339} & \underline{95.93} & 0.3618 & \underline{8.73} \\
    & \quad + \ourMethod        & \textbf{0.4978} & - & \underline{0.2227} & \textbf{0.7242} & \textbf{0.5471} & \textbf{95.38} & \underline{0.3311} & \textbf{8.82} \\
\midrule
\multirow{5}{*}{Depth} 
    & OminiControl              & - & \underline{904.9} & \textbf{0.2082} & \underline{0.6601} & \textbf{0.4470} & \textbf{101.88} & \textbf{0.3376} & 7.91 \\
    & \quad + Pissa             & - & 9416 & 0.1650 & 0.5233 & -0.0034 & 500.50 & 0.3140 & 8.46 \\
    & \quad + LoRA-GA           & - & 12160 & 0.0766 & 0.4959 & 0.1319 & 283.29 & 0.2333 & \textbf{9.67} \\
    & \quad + LoRA-One          & - & 1117 & 0.2022 & 0.6491 & \underline{0.3999} & \underline{119.16} & 0.3287 & 8.01 \\
    & \quad + \ourMethod        & - & \textbf{851.5} & \underline{0.2074} & \textbf{0.7587} & 0.3939 & 127.66 & \underline{0.3317} & \underline{9.45} \\
\midrule
\multirow{5}{*}{Deblur} 
    & OminiControl              & - & 89.39 & 0.2518 & 0.8083 & \underline{0.6660} & \underline{91.47} & 0.4835 & 15.99 \\
    & \quad + Pissa             & - & 91.43 & 0.2497 & 0.7932 & 0.6465 & 102.0 & 0.4731 & 15.64 \\
    & \quad + LoRA-GA           & - & 85.24 & 0.2517 & 0.8137 & 0.6575 & 91.70 & 0.4840 & 15.90 \\
    & \quad + LoRA-One          & - & \underline{69.07} & \underline{0.2528} & \underline{0.8203} & 0.6589 & 92.72 & \underline{0.4913} & \underline{16.53} \\
    & \quad + \ourMethod        & - & \textbf{55.87} & \textbf{0.2535} & \textbf{0.8328} & \textbf{0.7232} & \textbf{77.85} & \textbf{0.5224} & \textbf{17.32} \\
\bottomrule
\end{tabular}
}
\end{table}
\vspace{-1.0em}

%% file: chapters/5-expri.tex
\paragraph{Experimental Setting.}
We evaluate our method against sota methods across different image generation tasks: subject-driven generation, spatially-aligned generation, image deblurring, etc. To prove \ourMethod \ can be effective across models, we deploy \ourMethod \ on two mainstream models, FLUX-1 \cite{flux} and Stable Diffusion-3 \cite{SD3}.
Regarding dataset setting, we use the last 100,000 images in the MultiGen-20M \cite{Uni-ControlNet} dataset for training, and last 2,500 images in COCO for evaluation in Canny, Depth, and Deblur tasks.
For Dreambooth, we use the original 30 categories of images in training and generate 4 images with 25 prompts, with a total of 3,000 images for evaluation.
See more details about hyperparameter and other experiment settings in Appendix. \ref{sec:append_expr}.

We compare our method with other spectral-init methods, Pissa~\cite{pissa}, LoRA-GA~\cite{lora-ga}, and LoRA-One~\cite{lora-one}. In the spatial alignment task, we use OminiControl~\cite{tan2024ominicontrol} as the baseline method, which uses the vanilla LoRA initialization method, and deploy the above spectral-init methods.
\vspace{-0.5em}

\paragraph{Evaluation Metrics.}
We evaluate generated images along three axes. Controllability measures fidelity to the input condition: F1 score between Canny edges extracted from the generated and input images for the Canny task, and MSE between predicted and input depth maps for the Depth task. Alignment assesses semantic consistency via CLIP Text (text-image similarity), CLIP Image (image-image similarity), and DINO~\cite{dino} feature similarity against the reference. Image quality is reported in terms of FID~\cite{fid}, SSIM, and PSNR. Further details are provided in Appendix.~\ref{sec:append_expr}.
\vspace{-0.5em}

\begin{figure*}[t]
  \centering
  \includegraphics[width=\linewidth]{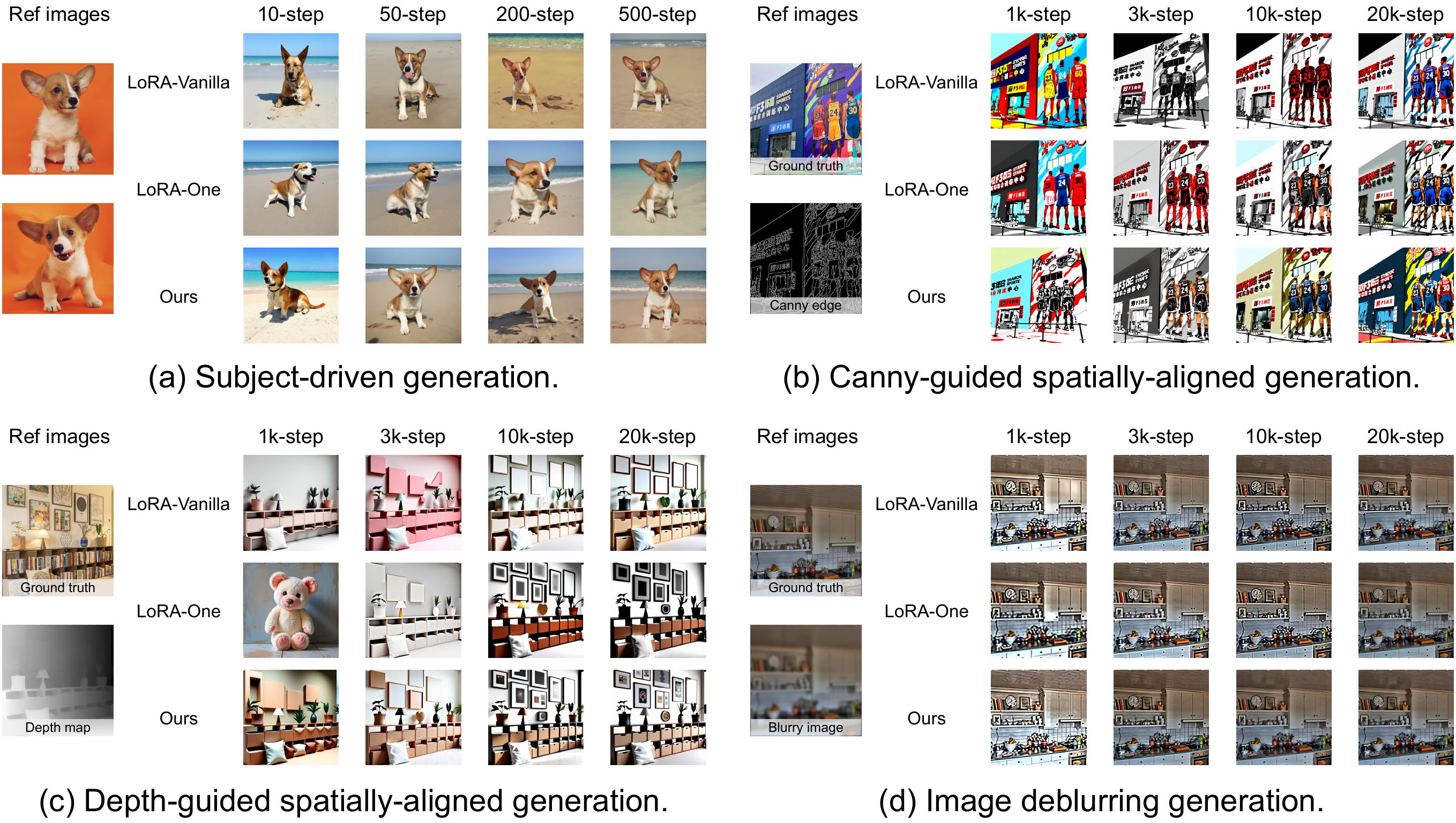}
  \caption{
  \textbf{Visual comparison with sota methods} across various text-to-image generation tasks based on LoRA fine-tuning. Our method more rapidly aligns with the target distribution, and optimizes the final generation performance. \textit{Caption: A sks dog on the beach. A high resolution picture of building. A high resolution picture of in-door decoration. A high resolution picture of kitchen}. 
  }
  \label{fig:vis}
  \vspace{-1.0em}
\end{figure*}

\subsection{Main Results}

\paragraph{Quantitative comparison.}

Table~\ref{tab:main_result} summarizes the performance across the four evaluated tasks. Our proposed PRISM-LoRA demonstrates superior controllability across all spatially-aligned and restoration tasks, achieving peak performance with an F1 score of 0.4978 on Canny, and minimizing MSE to 851.5 on Depth and 55.87 on Deblur. Notably, on the Deblur task, PRISM-LoRA consistently outperforms the OmniControl baseline across nearly all metrics, significantly improving FID from 91.47 to 77.85 and PSNR from 15.99 to 17.32. For subject-driven generation (DreamBooth), our method establishes robust leadership among LoRA variants in both semantic alignment (CLIP Text: 0.230, CLIP Image: 0.805, DINO: 0.682) and image quality (FID: 160.3), and achieve an average of 557 steps to reach best score, compared with LoRA-Base with 650, LoRA-GA 957 and LoRA-One 757 steps.
Furthermore, while improvements in alignment and fidelity on Canny and Depth are relatively modest, PRISM-LoRA remains highly competitive, consistently matching or surpassing existing spectral-init variants. It is also crucial to note that PiSSA experiences severe performance degradation on most tasks (e.g., yielding an FID of 456.62 on Canny and 500.50 on Depth). This stark contrast underscores the critical importance of gradient initialization quality and confirms that PRISM-LoRA provides a fundamentally more robust optimization starting point.


\paragraph{Qualitative comparison.}
Fig.~\ref{fig:vis} provides visual comparisons across the four downstream tasks at various training milestones. Across all settings, PRISM-LoRA exhibits a substantially accelerated convergence rate compared to both vanilla LoRA and LoRA-One. Specifically, coherent semantic structures emerge at remarkably early iterations (e.g., 10–500 steps), whereas the baselines continue to yield blurry, misaligned, or content-collapsed outputs at the equivalent stages. In the DreamBooth task, PRISM-LoRA faithfully preserves the unique identity of the reference subject (the specific dog) throughout the entire training trajectory. In contrast, vanilla LoRA suffers from concept drift toward generic dogs, and LoRA-One struggles to recover the subject identity until much later steps. For spatially-aligned generation (Canny and Depth), our model adheres more rigorously to the input conditioning signals, delivering sharper edge boundaries and highly accurate spatial layouts. Finally, in the image deblurring task, PRISM-LoRA not only restores high-frequency details earlier in the training process but also achieves perceptually cleaner reconstructions at convergence.
\vspace{-1.0em}

\begin{table}[ht]
\centering
\caption{\textbf{Ablation experiments} on Dreambooth and Canny-guided controllable generation.}
\label{tab:ablate_result}
\resizebox{0.75\textwidth}{!}{
\begin{tabular}{l|c|c|ccc}
\toprule
 
Tasks & Methods & F1$\uparrow$ & FID$\downarrow$ & CLIP-I$\uparrow$ & DINO$\uparrow$ \\
\midrule
\multirow{4}{*}{Dreambooth} 
    & LoRA-One                  & - & 169.1 & 0.795 & 0.667\\
    & w. principal timestep     & - & 161.7 & 0.798 & 0.677\\
    & w. golden channel         & - & 158.9 & 0.800 & 0.674\\
    & \ourMethod                & - & 160.3 & 0.805 & 0.682\\
\midrule
\multirow{4}{*}{Canny} 
    & LoRA-One                  & 0.4893 & 95.93 & 0.7214 & 0.5339\\
    & w. principal timestep     & 0.4834 & 95.82 & 0.7135 & 0.5351\\
    & w. golden channel         & 0.4912 & 97.14 & 0.7158 & 0.5387\\
    & \ourMethod                & 0.4978 & 95.38 & 0.7242 & 0.5471\\
\bottomrule

\end{tabular}
}
\end{table}
\vspace{-0.5em}

\subsection{Ablation Study}


To isolate the individual contributions of our proposed modules—Principal Timestep selection and Golden Channel filtering—we evaluate their independent and combined effects on the Dreambooth (subject-driven generation) and Canny (spatially-aligned generation) tasks. 
As observed in Table.~\ref{tab:ablate_result}, integrating either the principal timestep or the golden channel independently into the strong LoRA-One baseline yields overall performance improvements, particularly in terms of identity preservation and image quality. On the Dreambooth task, applying the Golden Channel alone improves the DINO score from 0.667 to 0.674 and substantially reduces FID, suggesting that filtering out task-irrelevant channels is crucial for maintaining optimization stability even under stochastic timestep sampling.

While individual components exhibit distinct merits, their combination in the full PRISM-LoRA framework achieves the optimal performance balance. Interestingly, on the Canny task, while applying independent modules causes slight fluctuations in specific metrics, the complete PRISM-LoRA clearly dominates, achieving the highest F1 score (0.4978) and DINO score (0.5471). This demonstrates a strong synergistic effect: suppressing the effective gradient rank (via Principal Timestep) creates a cleaner spectral foundation, which enables the sparse decomposition (via Golden Channel) to isolate the optimal fine-tuning directions more accurately. Together, they successfully mitigate the high-rank gradient dilemma inherent in flow-matching diffusion models.
\vspace{-0.5em}

%% file: chapters/6-concl.tex

\vspace{-0.5em}
In this paper, we identify and resolve a fundamental mismatch between the low-rank assumption of standard spectral initialization and the intrinsically high-rank gradient geometry found in flow-matching diffusion models. We demonstrate that stochastic timestep sampling induces directionally heterogeneous gradients, which renders conventional low-rank approximations lossy and misaligned. To overcome this barrier, we propose PRISM-LoRA, a principled spectral initialization framework specifically tailored for diffusion fine-tuning. By synergizing Principal Timestep selection to inherently suppress the effective gradient rank, and Golden Channel sparse decomposition to filter out task-irrelevant noise, PRISM-LoRA ensures that the initialized parameters robustly align with the long-horizon optimization trajectory.

\paragraph{Limitation and future works} While our approach achieves both faster convergence and strong performance on this task, a more comprehensive validation across diverse subject-driven scenarios remains an interesting direction for future work. On the theoretical side, our analysis primarily focuses on the observation that oversampling intermediate timesteps can be harmful to gradient quality. However, these timesteps also exhibit the largest variance, suggesting that they carry informative gradient signals that should be properly trained and accelerated rather than down-weighted. Extending our theory to characterize the spectral behavior across the full range of diffusion timesteps is a promising avenue for future investigation.

%% file: chapters/append.tex
\clearpage

\section*{Appendix}

\section{Preliminaries}
\label{sec:supp_prelim}
\input{chapters/3-Prelim}

\section{Detailed proof of principal timesteps}
\label{supp:sec:golden_timesteps}

\subsection{Problem Formulation}
\label{supp:sec:problem_formulation}
Consider the linear interpolation $x_t = (1-t)x_0 + t x_1$, where $x_0 \sim \mathcal{N}(0, I_d)$ is independent of $x_1$. Throughout this section, we adopt the \emph{Gaussian surrogate assumption} $x_1 \sim \mathcal{N}(0, \Sigma_1)$, which yields closed-form conditional covariances; for non-Gaussian $p_{\text{data}}$ the expressions below serve as the best linear-Gaussian approximation of $\operatorname{Var}[x_0 \mid x_t]$. Let the eigenvalues of $\Sigma_1$ be
\[
\lambda_1 \geq \lambda_2 \geq \cdots \geq \lambda_d > 0,
\]
and assume $\lambda_1 > \lambda_d$ (non-degenerate spectrum). The (rescaled) gradient covariance matrix is defined as
\begin{equation}
\label{eq:var_t}
    \Sigma_v(t) = \frac{1}{t^2}\,\operatorname{Var}[x_0 \mid x_t].
\end{equation}
To obtain Eq.~\ref{eq:var_t}, we need to first make a simple conversion for $x_1$.
According to the flow-matching identity, we can express $x_1$ as $x_1 = \frac{x_t - (1-t)x_0}{t}$.
As a result, the model prediction $v(x_t,t)$ can be express as 
\begin{equation}
    v(x_t,t)=\mathbb{E}[x_1-x_0 \mid x_t]=\mathbb{E} \left[ \frac{x_t - x_0}{t} \mid x_t \right] = \frac{x_t - \mathbb{E}[x_0 \mid x_t]}{t}.
\end{equation}
From this expression, we observe that the velocity field is determined by posterior expectation.
Finally, with the irreducible noise $x_1-x_0-v(x_t,t)=\frac{\mathbb{E}[x_0 \mid x_t]-x_0}{t}$, the covariance of the irreducible noise is $\operatorname{Cov} \left( \frac{\mathbb{E}[x_0 \mid x_t] - x_0}{t} \mid x_t \right) = \frac{1}{t^2} \operatorname{Var}[x_0 \mid x_t]$.

\paragraph{Derivation of the spectrum.}
Because $x_0 \perp x_1$, the marginal covariance of $x_t$ is
$\operatorname{Cov}(x_t) \;=\; (1-t)^2 I_d + t^2 \Sigma_1,$
and the cross-covariance is $\operatorname{Cov}(x_0, x_t) = (1-t) I_d$. Under the joint-Gaussian assumption, the conditional covariance is
\begin{equation}
\operatorname{Var}[x_0 \mid x_t] \;=\; I_d - (1-t)^2 \bigl[(1-t)^2 I_d + t^2 \Sigma_1\bigr]^{-1}.
\end{equation}
Diagonalizing in the eigenbasis of $\Sigma_1$, the $i$-th eigenvalue of $\operatorname{Var}[x_0 \mid x_t]$ equals $\frac{t^2 \lambda_i}{(1-t)^2 + t^2 \lambda_i}$, so the eigenvalues $\sigma_i(t)$ of $\Sigma_v(t)$ satisfy
\begin{equation}
\label{eq:sigma_i_t}
\sigma_i(t) \;=\; \frac{\lambda_i}{(1-t)^2 + t^2 \lambda_i}.
\end{equation}

\subsection{Spectral Comparison at $t$ in $\{0, 0.5, 1\}$}
\label{supp:sec:spectial_t_0}

Evaluating \eqref{eq:sigma_i_t} at characteristic timesteps:
\begin{itemize}
    \item \textbf{At $t \to 0$:} $\sigma_i(0) = \lambda_i$. The spectrum coincides with the data covariance $\Sigma_1$.
    \item \textbf{At $t = 0.5$:} $\sigma_i(0.5) = \dfrac{4\lambda_i}{1+\lambda_i}$.
    \item \textbf{At $t = 1$:} $\sigma_i(1) = 1$ for all $i$, so $\Sigma_v(1) = I_d$. Intuitively, when $x_t = x_1$ carries no information about $x_0$, the conditional variance reduces to the prior $I_d$ and the gradient direction is dominated entirely by noise.
\end{itemize}

\subsection{Majorization Lemma}
\label{supp:sec:spectial_t_1}
We first establish a general majorization result from which the three-point comparison (and the full-range monotonicity) will follow.

\begin{lemma}[Normalized-spectrum majorization]
\label{lem:majorization}
Let $\lambda = (\lambda_1, \dots, \lambda_d)$ with $\lambda_1 \geq \cdots \geq \lambda_d > 0$ and $\lambda_1 > \lambda_d$. Let $f:(0,\infty) \to (0,\infty)$ be strictly increasing, strictly concave, and satisfy $f(0^+) = 0$. Define the normalized spectra
\[
\hat{a}_i \;=\; \frac{\lambda_i}{\sum_j \lambda_j},
\qquad
\hat{b}_i \;=\; \frac{f(\lambda_i)}{\sum_j f(\lambda_j)}.
\]
Then $\hat{a} \succ \hat{b}$ (strictly), i.e.
\[
\sum_{i=1}^k \hat{a}_i \;\geq\; \sum_{i=1}^k \hat{b}_i \quad \forall k \in \{1,\dots,d\},
\]
with strict inequality for at least one $k < d$.
\end{lemma}

\begin{proof}
Both $\hat{a}$ and $\hat{b}$ sum to $1$, and since $f$ is strictly increasing, both are non-increasing in $i$ (same ordering). Consider the ratio
\begin{equation}
\label{eq:ratio}
\frac{\hat{a}_i}{\hat{b}_i} \;=\; \frac{\lambda_i}{f(\lambda_i)} \cdot \frac{\sum_j f(\lambda_j)}{\sum_j \lambda_j}.
\end{equation}
Strict concavity of $f$ together with $f(0^+)=0$ implies that $\lambda \mapsto \lambda/f(\lambda)$ is strictly increasing on $(0,\infty)$.\footnote{Indeed, $\frac{d}{d\lambda}\bigl(\lambda/f(\lambda)\bigr) = \bigl(f(\lambda)-\lambda f'(\lambda)\bigr)/f(\lambda)^2$; strict concavity and $f(0^+)=0$ give $f(\lambda) > \lambda f'(\lambda)$ for all $\lambda>0$.} Since $\lambda_i$ is non-increasing in $i$, the ratio $\hat{a}_i / \hat{b}_i$ is non-increasing in $i$; because $\lambda_1 > \lambda_d$, it is not constant.

\smallskip
\emph{Cut-property argument.} Since $\hat{a}_i/\hat{b}_i$ is non-increasing and non-constant, and $\sum_i \hat{a}_i = \sum_i \hat{b}_i = 1$, there exists a unique index $i^\star \in \{1,\dots,d-1\}$ such that
\[
\hat{a}_i \geq \hat{b}_i \text{ for } i \leq i^\star,
\qquad
\hat{a}_i \leq \hat{b}_i \text{ for } i > i^\star.
\]
For $k \leq i^\star$, each summand in $\sum_{i=1}^k(\hat{a}_i - \hat{b}_i)$ is non-negative, so the partial sum is $\geq 0$. For $k > i^\star$, using $\sum_{i=1}^d(\hat{a}_i - \hat{b}_i) = 0$,
\begin{equation}
    \sum_{i=1}^k(\hat{a}_i - \hat{b}_i) \;=\; -\!\!\sum_{i=k+1}^d(\hat{a}_i - \hat{b}_i) \;\geq\; 0,
\end{equation}

since each term in the tail sum is $\leq 0$. Hence $\sum_{i=1}^k \hat{a}_i \geq \sum_{i=1}^k \hat{b}_i$ for all $k$, with strict inequality at $k = i^\star$.
\end{proof}

\subsection{Full-Range Monotonicity of Spectral Flatness}
\label{supp:sec:transition-map}
We now promote the three-point comparison to full-range monotonicity in $t$.

\begin{lemma}[Monotone majorization along $t$]
\label{lem:full_range}
For any $0 \leq s < t \leq 1$, let $\hat{\sigma}(s)$ and $\hat{\sigma}(t)$ denote the normalized spectra of $\Sigma_v(s)$ and $\Sigma_v(t)$ respectively. Then $\hat{\sigma}(s) \succ \hat{\sigma}(t)$.
\end{lemma}

\begin{proof}
Define, for fixed $t \in [0,1]$, the map $g_t:(0,\infty) \to (0,\infty)$ by $g_t(\lambda) = \frac{\lambda}{(1-t)^2 + t^2 \lambda}$, so that $\sigma_i(t) = g_t(\lambda_i)$. A direct computation gives
\[
g_t'(\lambda) = \frac{(1-t)^2}{[(1-t)^2 + t^2\lambda]^2} > 0,
\qquad
g_t''(\lambda) = -\frac{2 t^2 (1-t)^2}{[(1-t)^2 + t^2\lambda]^3} \leq 0,
\]
so $g_t$ is strictly increasing and (for $t \in (0,1)$) strictly concave; moreover $g_t(0^+) = 0$.

For $s = 0$ and any $t \in (0,1]$, the claim $\hat{\sigma}(0) \succ \hat{\sigma}(t)$ follows immediately from Lemma~\ref{lem:majorization} applied with $f = g_t$.

For general $0 < s < t \leq 1$, write $\sigma_i(t) = (g_t \circ g_s^{-1})(\sigma_i(s))$. Let $h = g_t \circ g_s^{-1}$. Both $g_t$ and $g_s^{-1}$ are strictly increasing, so $h$ is strictly increasing; and the composition of a concave increasing function with an increasing concave inverse (i.e. with a convex $g_s^{-1}$ — note that the inverse of a concave increasing function is convex increasing) requires care. A direct calculation yields
\begin{equation}
\begin{aligned}
    h(\mu) &= g_t\!\left(\frac{(1-s)^2 \mu}{1 - s^2\mu}\right)  \\
    &= \frac{(1-s)^2\, \mu}{(1-t)^2(1 - s^2 \mu) + t^2(1-s)^2 \mu} \\
    &= \frac{(1-s)^2\, \mu}{(1-t)^2 + \bigl[t^2(1-s)^2 - s^2(1-t)^2\bigr]\mu}.
\end{aligned}
\end{equation}
Setting $\alpha = (1-s)^2 > 0$ and $\beta = t^2(1-s)^2 - s^2(1-t)^2$, we have $h(\mu) = \frac{\alpha\mu}{(1-t)^2 + \beta\mu}$. For $s < t \leq 1$ one checks that $\beta > 0$:\footnote{$\beta = [t(1-s) - s(1-t)][t(1-s) + s(1-t)] = (t-s)[t(1-s)+s(1-t)] > 0$ for $0 \leq s < t \leq 1$.} hence $h$ is of the same Möbius form $\frac{\alpha\mu}{c + \beta\mu}$ with $c, \alpha, \beta > 0$, so $h$ is strictly increasing, strictly concave, and $h(0^+) = 0$.

\smallskip
Applying Lemma~\ref{lem:majorization} to the spectrum $\sigma(s)$ with transformation $f = h$ gives $\hat{\sigma}(s) \succ \hat{\sigma}(t)$.
\end{proof}

\subsection{$r_{N}$ Flatness}

Define the $r_{N}$ metric
$
r_{N}(\hat{x})=\min{r : \sum_{i=1}^r \hat{x}_{(i)} \geq N/100 },$

where $\hat{x}_{(1)} \geq \hat{x}_{(2)} \geq \cdots$ is the sorted spectrum. A standard consequence of majorization is that if $\hat{a} \succ \hat{b}$ then $r_{N}(\hat{a}) \leq r_{N}(\hat{b})$, since larger partial sums reach the $N/100$ threshold no later.

Combining this with Lemma~\ref{lem:full_range} yields:

\begin{theorem}[Monotone spectral flattening]
\label{thm:r95_monotone}
Under the assumptions of Section~\ref{supp:sec:golden_timesteps} (Gaussian surrogate, $\lambda_1 > \lambda_d$), the $r_{N}$ metric is non-decreasing in $t$ on $[0,1]$: for all $0 \leq s < t \leq 1$,
$
r_{N}(\Sigma_v(s)) \leq r_{N}(\Sigma_v(t)).
$
In particular,
\begin{equation}
\label{eq:r95_chain}
r_{N}\bigl(\Sigma_v(0)\bigr) \leq r_{N}\bigl(\Sigma_v(0.5)\bigr)< r_{N}\bigl(\Sigma_v(1)\bigr) = \lceil N/100\, d \rceil,
\end{equation}
where the strict inequality $r_{N}(\Sigma_v(0.5)) < r_{N}(\Sigma_v(1))$ holds because $\hat{\sigma}(1) = (1/d,\dots,1/d)$ is the uniform distribution, which is strictly majorized by any non-uniform $\hat{\sigma}(0.5)$ whenever $\lambda_1 > \lambda_d$.
\end{theorem}

This proves that the gradient covariance becomes progressively flatter as $t$ moves from the data distribution toward the noise: small $t$ (the ``head'' of the sampling trajectory) admits a sharper low-rank structure in its gradient spectrum, whereas large $t$ (the ``tail'') approaches an isotropic spectrum and is inherently high-rank. This monotonicity motivates treating small-$t$ timesteps as the \emph{golden timesteps} for low-rank gradient-based initialization.

\section{Experimental Settings and Additional Results}
\label{sec:append_expr}
\paragraph{Evaluation Metrics}
We evaluate our method from three complementary perspectives: controllability, alignment, and image quality. For \textbf{controllability}, we measure how faithfully the generated images respect the input control signal. On the Canny task, we extract Canny edges from the generated images and compute the F1 score against the input edge map, where higher values indicate better edge preservation. On the Depth task, we estimate the depth map of the generated images using a pretrained depth predictor and report the Mean Squared Error (MSE) against the input depth condition, where lower values indicate more accurate geometric alignment.

For \textbf{alignment}, we assess semantic consistency from both textual and visual perspectives. \textit{CLIP Text} computes the cosine similarity between CLIP embeddings of the generated image and the text prompt, reflecting text-image semantic alignment. \textit{CLIP Image} measures the cosine similarity between CLIP embeddings of the generated and reference images, capturing high-level visual consistency. \textit{DINO} reports the cosine similarity in the DINO~\cite{dino} feature space, which is more sensitive to fine-grained structural and identity-level correspondence than CLIP.

For \textbf{image quality}, we adopt three widely-used metrics, and calculated by pyiqa~\cite{pyiqa}. \textit{FID} measures the distributional distance between generated and real images in the Inception feature space, with lower values indicating more realistic generations. \textit{SSIM} evaluates structural similarity between the generated image and the reference at the pixel level, accounting for luminance, contrast, and structural information. \textit{PSNR} measures pixel-wise reconstruction fidelity in decibels, where higher values denote less distortion. Together, these metrics provide a comprehensive assessment of perceptual realism, structural consistency, and reconstruction accuracy.

\subsection{Implementation Details}
In this section, we briefly introduce the datasets, training details, and evaluation metrics used in experiments.

\noindent\textbf{Hyperparameters setting}
We implement our method on top of two baselines. For the OminiControl~\cite{tan2024ominicontrol} baseline (Canny-to-image, depth-to-image, and deblurring), we follow the original setup: FLUX.1 as the base DiT, LoRA rank 4, effective batch size 8 (batch size 1 with gradient accumulation of 8), prodigy optimizer with weight decay 0.01, and $512\times512$ resolution. Training uses the last 300,000 images of text-to-image-2M, with Canny/depth maps and Gaussian blur generated on the fly. For the DreamBooth~\cite{ruiz2023dreambooth} baseline we use Stable Diffusion 3-medium with the default DreamBooth setting: 30 subjects with 3--5 reference images each, prompts of the form \textit{``a [sks] [class noun]''}, class-specific prior-preservation loss with $\lambda = 1$ ($\sim$1000 class samples), and learning rate $5\times10^{-6}$. Training iterations for each LoRA variant are reported in Table~\ref{tab:main_result}.

For our method hyperparameters, we use 1000 to 800 and 100 to 0 as our timestep selection to better fit the low-rank hypothesis. For golden channel, we set the $L_1$ normalized bound on 1, indicating the most sparse sampling situation. For rank selection after sparse iteration, we select the top-2 rank of singular values to remain instead of being set to 0.

We deploy our models on NVIDIA L40S GPU, with 48 gigabytes of VRAM. For dreambooth, the total training time is 25 minutes per class; for the subject-alignment task, we use 2 L40S and train for almost 20 hours.

%% file: chapters/3-prelim.tex
\subsection{Flow Matching Models}
Prior diffusion models based on denoising diffusion probabilistic models (DDPMs) \cite{ddpm} learn generative processes by gradually adding Gaussian noise to the data and training a model to reverse this stochastic diffusion process. In contrast, Flow Matching (FM) \cite{liu2022flow} formulates generative modeling as learning a continuous-time velocity field that deterministically transports samples from a simple prior distribution to the data distribution. Let $x_0 \sim p_0(x)$ and $x_1 \sim q(x)$ denote noise and data samples. The interpolation path is as follows:
\begin{equation}
\label{eq:prem_xt}
    x_t = (1 - t)x_0 + t x_1.
\end{equation}
As above, a vector field $v_\theta(x,t)$—parameterized by a neural network, is constructed to approximate a target vector field $v_t(x)$, which transports a simple prior distribution $p_0(x)$ to a complex data distribution $p_1(x)\approx q(x)$ via an ordinary differential equation (ODE): $\frac{d}{dt}\phi_t(x)=v_t(\phi_t(x))$ with $\phi_0(x)=x$. The ideal marginal FM objective is defined as:
\begin{equation}
    \mathcal{L}(\theta)=\mathbb{E}_{t\sim \mathcal{U}[0,1],x\sim p_t(x)}\left [\left \| v_\theta (x,t)-v_t(x) \right \|^2_2 \right].
\end{equation}

However, the marginal probability path $p_t(x)$ and the vector field $v_t(x)$ are generally unknown. To bypass this intractability, CFM \cite{lipman2022flow} constructs the target vector field by conditioning on individual data samples $x_1 \sim q(x_1)$. Given a conditional probability path $p_t(x|x_1)$ and its generating vector field $v_t(x|x_1)$, the tractable CFM objective is:
\begin{equation}
\begin{aligned}
    \mathcal{L}^*(\theta)&=\mathbb{E}_{t,x_1\sim q(x_1),x\sim p_t(x|x_1)}\left [\left \| v_\theta (x,t)-v_t(x|x_1) \right \|^2_2 \right] \\
    &=\mathbb{E}_{t,x1\sim q(x_1),x\sim p_t(x|x_1)}\left [\left \| v_\theta ((1-t)x_0+tx_1,t)-(x_1-x_0) \right \|^2_2 \right] 
\end{aligned}
\end{equation}
Minimizing $\mathcal{L}^*_\theta$ is mathematically equivalent to minimizing the intractable $\mathcal{L_\theta}$ up to a constant. The inference process solves the ODE backward from $t=0$ to $t=1$ through n iterative updates:
\begin{equation}
    x_{t_{i+1}} = x_{t_i} - (t_i-t_{i+1})v_\theta(x_{t_i}, t_i), \quad x_{t_0}\sim \mathcal{N}(0,I),
\end{equation}
with final output $x_{t_n}$ as the generated sample.

\subsection{Gradient-guided Parameter Efficient Fine-tuning}
Low-Rank Adaptation \cite{hu2022lora} is one of the parameter efficient fine-tuning methods that adapts large pre-trained models by injecting trainable low-rank matrices into existing weight layers. Given a pre-trained weight matrix $W \in \mathbb{R}^{d \times k}$, LoRA models its update as a low-rank decomposition:
\begin{equation}
    W{'}=W + \Delta W, \quad \Delta W=\eta BA,
\end{equation}
where $B \in \mathbb{R}^{d \times r}$ and $A \in \mathbb{R}^{r \times k}$, $r \ll \text{min}(d,k)$ denotes the rank, and $\eta$ is a scaling factor. This significantly reduces the number of trainable parameters and memory footprint.

Theoretically, the update direction of the LoRA module can be aligned with the full fine-tuning at the early training stage via the first gradient descent step, which can be mathematically expressed as:
\begin{equation}
    \eta \left( \Delta BA_\text{init} + B_\text{init}\Delta A \right) = \eta \lambda \left[\nabla_B \mathcal{L} (B_\text{init})A_\text{init} + B_\text{init}\nabla_A \mathcal{L}(A_\text{init}) \right],
\end{equation}
To measure its approximation quality of scaled the update of the weights in full fine-tuning $\zeta \Delta W=\zeta \lambda \nabla_W \mathcal{L}(W_0)$, the Frobenius norm of the difference between these two updates \cite{lora-ga,lora-one} is commonly used as a criterion:
\begin{equation}
\begin{aligned}
    &\left \| \eta  \left( \Delta BA_\text{init} + B_\text{init}\Delta A \right) - \zeta \lambda \nabla_W \mathcal{L}(W_0) \right \|_F \\
    = & \lambda \left \| \eta \nabla_B \mathcal{L}(B_\text{init})A_\text{init} + \eta B_\text{init}\nabla_A \mathcal{L}(A_\text{init}) - \zeta \nabla_W \mathcal{L}(W_0) \right \|_F \\
    = & \lambda \| \eta^2 \nabla_{W'} \mathcal{L}(W_0) \cdot A_{\text{init}}^T A_{\text{init}} + \eta^2 B_{\text{init}} B_{\text{init}}^T \cdot \nabla_W \mathcal{L}(W_0) - \zeta \nabla_W \mathcal{L}(W_0) \|_F .
\end{aligned}
\end{equation}
The svd decomposition can be formulated as $G^{\natural}=U_{G^\natural}S_{G^\natural}V_{G^\natural}$. Matrix $U$ and $V$ are orthogonal matrix, and $S$ is a strictly increasing sequence of singular value formulated as $S_{G^\natural} = \{\sigma_1, \sigma_2,...\sigma_d\}$.



%% file: checklist.tex
\section*{NeurIPS Paper Checklist}
\begin{enumerate}

\item {\bf Claims}
    \item[] Question: Do the main claims made in the abstract and introduction accurately reflect the paper's contributions and scope?
    \item[] Answer: \answerYes{} 
    \item[] Justification: We clearly include our paper's contributions and scope in the abstract, and detail them in the introduction.
    \item[] Guidelines:
    \begin{itemize}
        \item The answer \answerNA{} means that the abstract and introduction do not include the claims made in the paper.
        \item The abstract and/or introduction should clearly state the claims made, including the contributions made in the paper and important assumptions and limitations. A \answerNo{} or \answerNA{} answer to this question will not be perceived well by the reviewers. 
        \item The claims made should match theoretical and experimental results, and reflect how much the results can be expected to generalize to other settings. 
        \item It is fine to include aspirational goals as motivation as long as it is clear that these goals are not attained by the paper. 
    \end{itemize}

\item {\bf Limitations}
    \item[] Question: Does the paper discuss the limitations of the work performed by the authors?
    \item[] Answer: \answerYes{} 
    \item[] Justification: We discuss the limitations of our work from several aspects in the appendix.
    \item[] Guidelines:
    \begin{itemize}
        \item The answer \answerNA{} means that the paper has no limitation while the answer \answerNo{} means that the paper has limitations, but those are not discussed in the paper. 
        \item The authors are encouraged to create a separate ``Limitations'' section in their paper.
        \item The paper should point out any strong assumptions and how robust the results are to violations of these assumptions (e.g., independence assumptions, noiseless settings, model well-specification, asymptotic approximations only holding locally). The authors should reflect on how these assumptions might be violated in practice and what the implications would be.
        \item The authors should reflect on the scope of the claims made, e.g., if the approach was only tested on a few datasets or with a few runs. In general, empirical results often depend on implicit assumptions, which should be articulated.
        \item The authors should reflect on the factors that influence the performance of the approach. For example, a facial recognition algorithm may perform poorly when image resolution is low or images are taken in low lighting. Or a speech-to-text system might not be used reliably to provide closed captions for online lectures because it fails to handle technical jargon.
        \item The authors should discuss the computational efficiency of the proposed algorithms and how they scale with dataset size.
        \item If applicable, the authors should discuss possible limitations of their approach to address problems of privacy and fairness.
        \item While the authors might fear that complete honesty about limitations might be used by reviewers as grounds for rejection, a worse outcome might be that reviewers discover limitations that aren't acknowledged in the paper. The authors should use their best judgment and recognize that individual actions in favor of transparency play an important role in developing norms that preserve the integrity of the community. Reviewers will be specifically instructed to not penalize honesty concerning limitations.
    \end{itemize}

\item {\bf Theory assumptions and proofs}
    \item[] Question: For each theoretical result, does the paper provide the full set of assumptions and a complete (and correct) proof?
    \item[] Answer: \answerYes{} 
    \item[] Justification: This paper presents the hypothesis and partial proof in the main text and the full hypothesis and proof in the appendix.
    \item[] Guidelines:
    \begin{itemize}
        \item The answer \answerNA{} means that the paper does not include theoretical results. 
        \item All the theorems, formulas, and proofs in the paper should be numbered and cross-referenced.
        \item All assumptions should be clearly stated or referenced in the statement of any theorems.
        \item The proofs can either appear in the main paper or the supplemental material, but if they appear in the supplemental material, the authors are encouraged to provide a short proof sketch to provide intuition. 
        \item Inversely, any informal proof provided in the core of the paper should be complemented by formal proofs provided in appendix or supplemental material.
        \item Theorems and Lemmas that the proof relies upon should be properly referenced. 
    \end{itemize}

    \item {\bf Experimental result reproducibility}
    \item[] Question: Does the paper fully disclose all the information needed to reproduce the main experimental results of the paper to the extent that it affects the main claims and/or conclusions of the paper (regardless of whether the code and data are provided or not)?
    \item[] Answer: \answerYes{} 
    \item[] Justification: We fully disclose the conditions to reproduce our results in the paper.
    \item[] Guidelines:
    \begin{itemize}
        \item The answer \answerNA{} means that the paper does not include experiments.
        \item If the paper includes experiments, a \answerNo{} answer to this question will not be perceived well by the reviewers: Making the paper reproducible is important, regardless of whether the code and data are provided or not.
        \item If the contribution is a dataset and\slash or model, the authors should describe the steps taken to make their results reproducible or verifiable. 
        \item Depending on the contribution, reproducibility can be accomplished in various ways. For example, if the contribution is a novel architecture, describing the architecture fully might suffice, or if the contribution is a specific model and empirical evaluation, it may be necessary to either make it possible for others to replicate the model with the same dataset, or provide access to the model. In general. releasing code and data is often one good way to accomplish this, but reproducibility can also be provided via detailed instructions for how to replicate the results, access to a hosted model (e.g., in the case of a large language model), releasing of a model checkpoint, or other means that are appropriate to the research performed.
        \item While NeurIPS does not require releasing code, the conference does require all submissions to provide some reasonable avenue for reproducibility, which may depend on the nature of the contribution. For example
        \begin{enumerate}
            \item If the contribution is primarily a new algorithm, the paper should make it clear how to reproduce that algorithm.
            \item If the contribution is primarily a new model architecture, the paper should describe the architecture clearly and fully.
            \item If the contribution is a new model (e.g., a large language model), then there should either be a way to access this model for reproducing the results or a way to reproduce the model (e.g., with an open-source dataset or instructions for how to construct the dataset).
            \item We recognize that reproducibility may be tricky in some cases, in which case authors are welcome to describe the particular way they provide for reproducibility. In the case of closed-source models, it may be that access to the model is limited in some way (e.g., to registered users), but it should be possible for other researchers to have some path to reproducing or verifying the results.
        \end{enumerate}
    \end{itemize}

\item {\bf Open access to data and code}
    \item[] Question: Does the paper provide open access to the data and code, with sufficient instructions to faithfully reproduce the main experimental results, as described in supplemental material?
    \item[] Answer: \answerYes{} 
    \item[] Justification: The repository link mentioned in the abstract can refer to our code. We also provide the scripts to download or generate the data we used in the paper.
    \item[] Guidelines:
    \begin{itemize}
        \item The answer \answerNA{} means that paper does not include experiments requiring code.
        \item Please see the NeurIPS code and data submission guidelines (\url{https://neurips.cc/public/guides/CodeSubmissionPolicy}) for more details.
        \item While we encourage the release of code and data, we understand that this might not be possible, so \answerNo{} is an acceptable answer. Papers cannot be rejected simply for not including code, unless this is central to the contribution (e.g., for a new open-source benchmark).
        \item The instructions should contain the exact command and environment needed to run to reproduce the results. See the NeurIPS code and data submission guidelines (\url{https://neurips.cc/public/guides/CodeSubmissionPolicy}) for more details.
        \item The authors should provide instructions on data access and preparation, including how to access the raw data, preprocessed data, intermediate data, and generated data, etc.
        \item The authors should provide scripts to reproduce all experimental results for the new proposed method and baselines. If only a subset of experiments are reproducible, they should state which ones are omitted from the script and why.
        \item At submission time, to preserve anonymity, the authors should release anonymized versions (if applicable).
        \item Providing as much information as possible in supplemental material (appended to the paper) is recommended, but including URLs to data and code is permitted.
    \end{itemize}

\item {\bf Experimental setting/details}
    \item[] Question: Does the paper specify all the training and test details (e.g., data splits, hyperparameters, how they were chosen, type of optimizer) necessary to understand the results?
    \item[] Answer: \answerTODO{} 
    \item[] Justification: \justificationTODO{}
    \item[] Guidelines:
    \begin{itemize}
        \item The answer \answerNA{} means that the paper does not include experiments.
        \item The experimental setting should be presented in the core of the paper to a level of detail that is necessary to appreciate the results and make sense of them.
        \item The full details can be provided either with the code, in appendix, or as supplemental material.
    \end{itemize}

\item {\bf Experiment statistical significance}
    \item[] Question: Does the paper report error bars suitably and correctly defined or other appropriate information about the statistical significance of the experiments?
    \item[] Answer: \answerYes{} 
    \item[] Justification: For small dataset, the results can be reproduced through a different server we used. For large dataset and task, we use visualization results of different steps to show the quality of generated images.
    \item[] Guidelines:
    \begin{itemize}
        \item The answer \answerNA{} means that the paper does not include experiments.
        \item The authors should answer \answerYes{} if the results are accompanied by error bars, confidence intervals, or statistical significance tests, at least for the experiments that support the main claims of the paper.
        \item The factors of variability that the error bars are capturing should be clearly stated (for example, train/test split, initialization, random drawing of some parameter, or overall run with given experimental conditions).
        \item The method for calculating the error bars should be explained (closed form formula, call to a library function, bootstrap, etc.)
        \item The assumptions made should be given (e.g., Normally distributed errors).
        \item It should be clear whether the error bar is the standard deviation or the standard error of the mean.
        \item It is OK to report 1-sigma error bars, but one should state it. The authors should preferably report a 2-sigma error bar than state that they have a 96\% CI, if the hypothesis of Normality of errors is not verified.
        \item For asymmetric distributions, the authors should be careful not to show in tables or figures symmetric error bars that would yield results that are out of range (e.g., negative error rates).
        \item If error bars are reported in tables or plots, the authors should explain in the text how they were calculated and reference the corresponding figures or tables in the text.
    \end{itemize}

\item {\bf Experiments compute resources}
    \item[] Question: For each experiment, does the paper provide sufficient information on the computer resources (type of compute workers, memory, time of execution) needed to reproduce the experiments?
    \item[] Answer: \answerYes{} 
    \item[] Justification: We disclose the details of our training server in the appendix.
    \item[] Guidelines:
    \begin{itemize}
        \item The answer \answerNA{} means that the paper does not include experiments.
        \item The paper should indicate the type of compute workers CPU or GPU, internal cluster, or cloud provider, including relevant memory and storage.
        \item The paper should provide the amount of compute required for each of the individual experimental runs as well as estimate the total compute. 
        \item The paper should disclose whether the full research project required more compute than the experiments reported in the paper (e.g., preliminary or failed experiments that didn't make it into the paper). 
    \end{itemize}
    
\item {\bf Code of ethics}
    \item[] Question: Does the research conducted in the paper conform, in every respect, with the NeurIPS Code of Ethics \url{https://neurips.cc/public/EthicsGuidelines}?
    \item[] Answer: \answerYes{} 
    \item[] Justification: The whole paper follows the NeurIPS code of ethics.
    \item[] Guidelines:
    \begin{itemize}
        \item The answer \answerNA{} means that the authors have not reviewed the NeurIPS Code of Ethics.
        \item If the authors answer \answerNo, they should explain the special circumstances that require a deviation from the Code of Ethics.
        \item The authors should make sure to preserve anonymity (e.g., if there is a special consideration due to laws or regulations in their jurisdiction).
    \end{itemize}

\item {\bf Broader impacts}
    \item[] Question: Does the paper discuss both potential positive societal impacts and negative societal impacts of the work performed?
    \item[] Answer: \answerYes{} 
    \item[] Justification: We fully discuss both positive and negative societal impacts of our work.
    \item[] Guidelines:
    \begin{itemize}
        \item The answer \answerNA{} means that there is no societal impact of the work performed.
        \item If the authors answer \answerNA{} or \answerNo, they should explain why their work has no societal impact or why the paper does not address societal impact.
        \item Examples of negative societal impacts include potential malicious or unintended uses (e.g., disinformation, generating fake profiles, surveillance), fairness considerations (e.g., deployment of technologies that could make decisions that unfairly impact specific groups), privacy considerations, and security considerations.
        \item The conference expects that many papers will be foundational research and not tied to particular applications, let alone deployments. However, if there is a direct path to any negative applications, the authors should point it out. For example, it is legitimate to point out that an improvement in the quality of generative models could be used to generate Deepfakes for disinformation. On the other hand, it is not needed to point out that a generic algorithm for optimizing neural networks could enable people to train models that generate Deepfakes faster.
        \item The authors should consider possible harms that could arise when the technology is being used as intended and functioning correctly, harms that could arise when the technology is being used as intended but gives incorrect results, and harms following from (intentional or unintentional) misuse of the technology.
        \item If there are negative societal impacts, the authors could also discuss possible mitigation strategies (e.g., gated release of models, providing defenses in addition to attacks, mechanisms for monitoring misuse, mechanisms to monitor how a system learns from feedback over time, improving the efficiency and accessibility of ML).
    \end{itemize}
    
\item {\bf Safeguards}
    \item[] Question: Does the paper describe safeguards that have been put in place for responsible release of data or models that have a high risk for misuse (e.g., pre-trained language models, image generators, or scraped datasets)?
    \item[] Answer: \answerNA{} 
    \item[] Justification: Our released models do not have the risk for misuse and some safeguard risks.
    \item[] Guidelines:
    \begin{itemize}
        \item The answer \answerNA{} means that the paper poses no such risks.
        \item Released models that have a high risk for misuse or dual-use should be released with necessary safeguards to allow for controlled use of the model, for example by requiring that users adhere to usage guidelines or restrictions to access the model or implementing safety filters. 
        \item Datasets that have been scraped from the Internet could pose safety risks. The authors should describe how they avoided releasing unsafe images.
        \item We recognize that providing effective safeguards is challenging, and many papers do not require this, but we encourage authors to take this into account and make a best faith effort.
    \end{itemize}

\item {\bf Licenses for existing assets}
    \item[] Question: Are the creators or original owners of assets (e.g., code, data, models), used in the paper, properly credited and are the license and terms of use explicitly mentioned and properly respected?
    \item[] Answer: \answerYes{} 
    \item[] Justification: We mention and respect the license and terms of use of the creators of code, data, and models related to this paper.
    \item[] Guidelines:
    \begin{itemize}
        \item The answer \answerNA{} means that the paper does not use existing assets.
        \item The authors should cite the original paper that produced the code package or dataset.
        \item The authors should state which version of the asset is used and, if possible, include a URL.
        \item The name of the license (e.g., CC-BY 4.0) should be included for each asset.
        \item For scraped data from a particular source (e.g., website), the copyright and terms of service of that source should be provided.
        \item If assets are released, the license, copyright information, and terms of use in the package should be provided. For popular datasets, \url{paperswithcode.com/datasets} has curated licenses for some datasets. Their licensing guide can help determine the license of a dataset.
        \item For existing datasets that are re-packaged, both the original license and the license of the derived asset (if it has changed) should be provided.
        \item If this information is not available online, the authors are encouraged to reach out to the asset's creators.
    \end{itemize}

\item {\bf New assets}
    \item[] Question: Are new assets introduced in the paper well documented and is the documentation provided alongside the assets?
    \item[] Answer: \answerYes{} 
    \item[] Justification: Our code and new form of data all have detailed documentation provided with the new assets.
    \item[] Guidelines:
    \begin{itemize}
        \item The answer \answerNA{} means that the paper does not release new assets.
        \item Researchers should communicate the details of the dataset\slash code\slash model as part of their submissions via structured templates. This includes details about training, license, limitations, etc. 
        \item The paper should discuss whether and how consent was obtained from people whose asset is used.
        \item At submission time, remember to anonymize your assets (if applicable). You can either create an anonymized URL or include an anonymized zip file.
    \end{itemize}

\item {\bf Crowdsourcing and research with human subjects}
    \item[] Question: For crowdsourcing experiments and research with human subjects, does the paper include the full text of instructions given to participants and screenshots, if applicable, as well as details about compensation (if any)? 
    \item[] Answer: \answerNA{} 
    \item[] Justification: This paper does not involve anything related to crowdsourcing nor research with human subjects.
    \item[] Guidelines:
    \begin{itemize}
        \item The answer \answerNA{} means that the paper does not involve crowdsourcing nor research with human subjects.
        \item Including this information in the supplemental material is fine, but if the main contribution of the paper involves human subjects, then as much detail as possible should be included in the main paper. 
        \item According to the NeurIPS Code of Ethics, workers involved in data collection, curation, or other labor should be paid at least the minimum wage in the country of the data collector. 
    \end{itemize}

\item {\bf Institutional review board (IRB) approvals or equivalent for research with human subjects}
    \item[] Question: Does the paper describe potential risks incurred by study participants, whether such risks were disclosed to the subjects, and whether Institutional Review Board (IRB) approvals (or an equivalent approval/review based on the requirements of your country or institution) were obtained?
    \item[] Answer: \answerNA{} 
    \item[] Justification: This paper does not involve anything related to crowdsourcing nor research with human subjects.
    \item[] Guidelines:
    \begin{itemize}
        \item The answer \answerNA{} means that the paper does not involve crowdsourcing nor research with human subjects.
        \item Depending on the country in which research is conducted, IRB approval (or equivalent) may be required for any human subjects research. If you obtained IRB approval, you should clearly state this in the paper. 
        \item We recognize that the procedures for this may vary significantly between institutions and locations, and we expect authors to adhere to the NeurIPS Code of Ethics and the guidelines for their institution. 
        \item For initial submissions, do not include any information that would break anonymity (if applicable), such as the institution conducting the review.
    \end{itemize}

\item {\bf Declaration of LLM usage}
    \item[] Question: Does the paper describe the usage of LLMs if it is an important, original, or non-standard component of the core methods in this research? Note that if the LLM is used only for writing, editing, or formatting purposes and does \emph{not} impact the core methodology, scientific rigor, or originality of the research, declaration is not required.
    \item[] Answer: \answerNo{} 
    \item[] Justification: The core method and analysis are done by carefully researchers themselves.
    \item[] Guidelines:
    \begin{itemize}
        \item The answer \answerNA{} means that the core method development in this research does not involve LLMs as any important, original, or non-standard components.
        \item Please refer to our LLM policy in the NeurIPS handbook for what should or should not be described.
    \end{itemize}

\end{enumerate}